\documentclass[sigconf, nonacm]{acmart}

\usepackage{times}
\usepackage{latexsym}

\usepackage[utf8]{inputenc}

\usepackage{microtype}
\usepackage{amsmath}
\usepackage{amsfonts}
\usepackage{mathtools}
\usepackage{amsthm}
\usepackage{colortbl}
\usepackage{xcolor}
\usepackage{booktabs}
\usepackage{multirow}
\usepackage{hyperref}
\usepackage{natbib}

\newtheorem{theorem}{Theorem}[section]
\newtheorem{definition}{Definition}

\newcommand{\tensor}[1]{\mathcal{#1}}

\newcommand{\T}{\tensor{T}}

\newcommand{\R}{\mathbb{R}}

\AtBeginDocument{%
  }

\begin{document}

\title{Higher Order Transformers: Enhancing Stock Movement Prediction On Multimodal Time-Series Data}

\author{Soroush Omranpour}
\affiliation{
  \institution{Mila \& McGill University}
  \city{Montreal}
  \country{Canada}
}
\email{soroush.omranpour@mila.quebec}

\author{Guillaume Rabusseau}
\affiliation{
  \institution{Mila \& University of Montreal}
  \city{Montreal}
  \country{Canada}
}
\email{rabussgu@mila.quebec}

\author{Reihaneh Rabbany}
\affiliation{%
  \institution{Mila \& McGill University}
  \city{Montreal}
  \country{Canada}
}
\email{reihaneh.rabbany@mila.quebec}

\begin{abstract}
In this paper, we tackle the challenge of predicting stock movements in financial markets by introducing Higher Order Transformers, a novel architecture designed for processing multivariate time-series data. We extend the self-attention mechanism and the transformer architecture to a higher order, effectively capturing complex market dynamics across time and variables. To manage computational complexity, we propose a low-rank approximation of the potentially large attention tensor using tensor decomposition and employ kernel attention, reducing complexity to linear with respect to the data size. Additionally, we present an encoder-decoder model that integrates technical and fundamental analysis, utilizing multimodal signals from historical prices and related tweets. Our experiments on the Stocknet dataset demonstrate the effectiveness of our method, highlighting its potential for enhancing stock movement prediction in financial markets.
\end{abstract}


\maketitle

\section{Introduction}
Predicting stock movements in financial markets is of paramount importance for investors and traders alike, as it enables informed decision-making and enhances profitability. However, this task is inherently challenging due to the stochastic nature of market dynamics, the non-stationarity of stock prices, and the influence of numerous factors beyond historical prices, such as social media sentiment and inter-stock correlations.

Traditional approaches in stock prediction have primarily focused on technical analysis (TA) and fundamental analysis (FA), leveraging historical price data and key financial metrics, respectively \cite{yin2016predicting}. While these methods have provided valuable insights, they often fail to capture the complex interdependencies and the high-dimensional structure of financial data \cite{bollerslev1986generalized}. Recent advancements in machine learning, particularly in natural language processing and graph neural networks, have begun to address these limitations by integrating multimodal data sources, such as news articles and social media sentiment, thereby offering a more nuanced understanding of market dynamics \cite{hu2018listening, kim2019hats}.

Despite these advancements, existing models still struggle with the sheer volume and variability of financial data, often resulting in suboptimal predictive performance when dealing with high-dimensional, multivariate time-series data. To address these challenges, we introduce a novel architecture called Higher Order Transformers. This architecture extends the traditional transformer model by incorporating higher-order data structures in the self-attention mechanism, enabling it to capture more complex interrelationships across time and variables.

The contributions of this paper can be summarized as follows:
\begin{itemize}
\item We introduce \emph{Higher-Order Transformers}, a novel architecture tailored for processing higher-order data (with more than two dimensions where the last dimension works as the feature dimension).
\item We propose a \emph{Multimodal Encoder-Decoder} architecture based on the proposed Higher Order Transformer which integrates both news text and historical price data to provide a holistic understanding of the market dynamics.
\item Following a series of experiments and ablation studies, we demonstrate the effectiveness of the proposed method empirically.
\end{itemize}

\section{Related Work}
Predicting stock movements has traditionally utilized technical analysis (TA), focusing on historical price data and macroeconomic indicators, with common methods like GARCH and neural networks being prevalent for their ability to identify temporal patterns \cite{yin2016predicting, bollerslev1986generalized, yildirim2021forecasting}. However, such methods often fail to account for external factors that can significantly influence market dynamics, thus limiting their overall predictive scope.

Fundamental analysis (FA) attempts to fill this gap by integrating broader market elements such as investor sentiment and news, with the help of advancements in natural language processing (NLP). NLP has enabled more effective sentiment analysis from diverse data sources like news articles and social media, thereby enhancing the FA approaches \cite{kalyani2016stock, ding2015deep, abarbanell1997fundamental, hu2018listening, xu2018stock}. Innovative models such as the Hybrid Attention Networks (HAN) and StockNet have emerged, which blend attention mechanisms and variational auto-encoders to analyze text alongside price data, though the assumption of independent stock movements often hampers their effectiveness \cite{hu2019listening, xu2018stock}.

Graph Neural Networks (GNNs) have been introduced as a solution to address the interconnected nature of financial markets, by structuring market data into graph formats where nodes represent companies, allowing for enhanced data representation and learning through contextual inter-node relationships \cite{feng2019temporal, sawhney2020spatiotemporal, matsunaga2019exploring, kim2019hats, yangjia-etal-2022-fundamental}. These networks are particularly effective when combined with attention mechanisms and have been expanded upon with financial knowledge graphs to incorporate domain-specific knowledge \cite{kertkeidkachorn2023finkg}.

Despite progress in graph-based and NLP-enhanced stock prediction methods, fully integrated multimodal approaches that leverage both textual data and inter-stock relationships are still in their infancy. Recent proposals for such multimodal models aim to harness these diverse data sets to improve the predictive accuracy of stock movement models significantly \cite{daiya2021stock, li2020multimodal, sawhney2020deep, kim2019hats, liu-etal-2019-transformer, ZHANG2022117239, xie2023wall}. These efforts underscore ongoing challenges and the potential for future advancements in capturing complex market signals and correlations.

\section{Problem Definition}

In line with previous studies \cite{xu2018stock, sawhney2020deep}, we define the stock movement prediction task as a binary classification problem. For a given stock \( s \), the price movement from day \( t \) to \( t+1 \) is defined as follows:
\begin{equation}
    Y_t = 
    \begin{cases}
      0 & \text{if } p^c_t < p^c_{t-1} \\
      1 & \text{otherwise}
    \end{cases}
\end{equation}
where \( p^c_t \) represents the adjusted closing price\footnote{\url{https://www.investopedia.com/terms/a/adjusted_closing_price.asp}} on day \( t \). Here, \( Y_t = 0 \) indicates that the stock price has decreased, and \( Y_t = 1 \) indicates that the stock price has increased.

The objective of this task is to predict the price movement \( Y_{T+1} \) of a stock \( s \) based on its historical price data and related tweets within a time window of \( T \) days.\\

\section{Preliminary}
Let us define some notations. We use lower-case letters for vectors (\textit{e.g.,} $v$), upper-case letters for matrices (\textit{e.g.,} $M$), and calligraphic letters for tensors (\textit{e.g.,}$\tensor{T}$). We use $\otimes$ to denote the Kronecker product and $\times_i$ for the tensor product along mode $i$. The notation $[k]$ denotes the set $\{1,\cdots,k\}$ for any integer $k$. 

\begin{definition}[Tensor]
A $k$-th order tensor $\tensor{T}\in \mathbb{R}^{N_1\times N_2 \times ... \times N_k}$ can simply be seen as a multidimensional array.
\end{definition}

\begin{definition}[Tensor Mode and Fibers]
The mode-$i$ fibers  of $\tensor{T}$ are vectors obtained by fixing all indices except the $i$-th one: $\tensor{T}_{n_1,n_2,...,n_{i-1},:,n_{i+1},...,n_k}\in \R^{N_i}$.
\end{definition}

\begin{definition}[Tensor slice]
A tensor slice is a two-dimensional section (fragment) of a tensor, obtained by fixing all indexes except for two indices.
\end{definition}

\begin{definition}[Tensor Matricization]
The $i$-th mode matricization of a tensor is the matrix having its mode-$i$ fibers as columns and is denoted by $\T_{(i)}$, e.g., $\T_{(1)}\in \R^{N_1\times N_2\cdots N_k}$.
\end{definition}

\begin{definition}[Kronecker Product]
The Kronecker product, denoted by \(\otimes\), is a binary operation that combines two matrices or tensors to form a larger matrix or tensor. For two matrices \(A \in \R^{m \times n}\) and \(B \in \R^{p \times q}\), their Kronecker product \(A \otimes B\) is defined as:
\[
A \otimes B = \begin{bmatrix}
a_{11}B & \cdots & a_{1n}B \\
\vdots & \ddots & \vdots \\
a_{m1}B & \cdots & a_{mn}B
\end{bmatrix}
\in \mathbb{R}^{mp \times nq}
\]
\end{definition}

In this paper, we consider multivariate timeseries data in the form of third-order tensors $\tensor{X} \in \R^{N \times T \times d}$ where $N$ is the number of variables, $T$ is the number of timesteps, and $d$ is the number of hidden/feature dimensions.  \\

\section{Method}
\subsection{Tokenization}
In this section, we explain the process of tokenizing the input multivariate time-series data. Following prior work \cite{sawhney2020deep}, we construct a price vector for each stock \(i\) at each day \(t\) in the form of \(x_{i,t} = [p^c_{i,t}, p^h_{i,t}, p^l_{i,t}]\), comprising the stock's adjusted closing price, highest price, and lowest price. 

In line with the multivariate time-series forecasting literature \cite{rasul2024lagllama, zhou2021informer}, we also found it beneficial to add date features, including the day of the month, month of the year, and year. The combination of price and date features forms a six-dimensional vector for each stock on each day.

Inspired by \cite{darcet2024vision}, we add stock-specific learnable tokens to the beginning of each time-series and treat them as the common CLS token in transformer encoders. Similar to BERT \cite{devlin2019bert} and ViT \cite{dosovitskiy2021image}, we used the hidden state of this special token as the stock representation over the whole time window for the classification task.\\

\subsection{Higher Order Transformer}
In this section, we first review the self-attention mechanism in Transformer layers \cite{vaswani2023attention}. Then, we extend it to higher orders by tensorizing queries, keys, and values, thereby formulating higher order transformer layers. Given that computing attention on tensors is prohibitively costly, we propose a low-rank approximation using Kronecker decomposition. Additionally, we incorporate the attention kernel trick \cite{choromanski2022rethinking} to significantly reduce the computational complexity.\\

\subsubsection{Standard Transformer Layer\\}
A Transformer encoder layer comprises two primary components: a self-attention layer \(g_\text{Attn} : \mathbb{R}^{n \times d} \to \mathbb{R}^{n \times d}\) and an elementwise feedforward layer \(g_\text{MLP} : \mathbb{R}^{n \times d} \to \mathbb{R}^{n \times d}\). For a set of \(n\) input vectors \(X \in \mathbb{R}^{n \times d}\), a Transformer layer computes the following\footnote{We omitted normalization after \(g_\text{Attn}(\cdot)\) and \(g_\text{MLP}(\cdot)\) for simplicity.}:

\begin{align}
    g_\text{Attn}(X)_i &= X_i + \sum_{h=1}^H \sum_{j=1}^n S^h_{ij} X_j W_h^V W_h^O, \\
    f_\text{Enc}(X)_i &= g_\text{Attn}(X)_i + g_\text{MLP}(g_\text{Attn}(X))_i,
\end{align}

where \(i\) denotes the row index of the matrix, \(H\) is the number of attention heads, and \(d_H\) is the head size. The weight matrices \(W_h^V, W_h^K, W_h^Q \in \mathbb{R}^{d \times d_H}\) and \(W_h^O \in \mathbb{R}^{d_H \times d}\) are associated with each attention head. The attention coefficients are computed as:

\begin{equation}
    S^h = \text{Softmax}\left(\frac{X W_h^Q (X W_h^K)^T}{\sqrt{d}}\right),
\end{equation}
where the Softmax function is applied row-wise.

\subsubsection{Higher Order Transformer Layer\\}
To generalize the formulation of scaled dot-product attention to higher-order input data \(\tensor{X} \in \mathbb{R}^{N \times T \times d}\), we extend the concept of attention to tensors. We define the transformer encoder layer as a function \(f_{\text{Enc}}: \mathbb{R}^{N \times T \times d} \to \mathbb{R}^{N \times T \times d}\), formulated as follows:

\begin{align}
    g_\text{Attn}(\tensor{X})_{it} &= \tensor{X}_{it} + \sum_{h=1}^H \sum_{j, \tau} \tensor{S}^h_{ijt\tau} W_h^V W_h^O \tensor{X}_{j\tau}, \\
    f_\text{Enc}(\tensor{X})_{it} &= g_\text{Attn}(\tensor{X})_{it} + g_\text{MLP}(g_\text{Attn}(\tensor{X}))_{it},
\end{align}
where \(\tensor{X}_{it} \in \mathbb{R}^{d}\) denotes the third-mode fiber of \(\tensor{X}\) corresponding to variable \(i\) at timestep \(t\). The attention scores are computed as:

\begin{equation}
\label{eq:high_order_attn}
    \tensor{S}^h_{ijt\tau} = \frac{e^{z^h_{ijt\tau}}}{\sum_{j,\tau} e^{z^h_{ijt\tau}}},
\end{equation}
where $\tensor{S}^h_{ijt\tau}$ denotes the attention score for variable $i$ at time $t$ attending to variable $j$ at time $\tau$ on head $h$ and \(z^h_{ijt\tau}\) is computed as:
\begin{equation}
    z^h_{ijt\tau} = \frac{(W_h^Q \tensor{X}_{it})^T W_h^K \tensor{X}_{j\tau}}{\sqrt{d}}.
\end{equation}

Computing the attention tensor \(\tensor{S} \in \mathbb{R}^{N \times N \times T \times T}\) requires \(\mathcal{O}(N^2 T^2 d)\) operations, which can be computationally expensive in most cases.\\

\subsubsection{Low-Rank Approximation With Kronecker Decomposition\\}

Consider the attention matrix \( S \in \mathbb{R}^{NT \times NT} \) as the result of reshaping the tensor \(\tensor{S}\) by flattening the variable and time dimensions. We parameterize the attention matrix \( S \) using a rank \( R \) Kronecker decomposition with factor matrices \( S^{(1)}_i \in \mathbb{R}^{N \times N} \) and \( S^{(2)}_i \in \mathbb{R}^{T \times T} \):
\begin{align}
    \label{eq:att_approx}
    S \approx \sum_{i=1}^R (S^{(1)}_i \otimes S^{(2)}_i),
\end{align}
where \( S^{(1)}_i \) and \( S^{(2)}_i \) represent lower-order attention matrices over the variable and time dimensions, respectively. Having \( R \) attention matrices acts similarly to multi-head attention with \( R \) heads. In practice, we adopt the multi-head attention mechanism instead of the summation shown in Equation \eqref{eq:att_approx}, as it uses more parameters and can be more expressive.

\begin{figure}
    \centering
    \includegraphics[width=\linewidth]{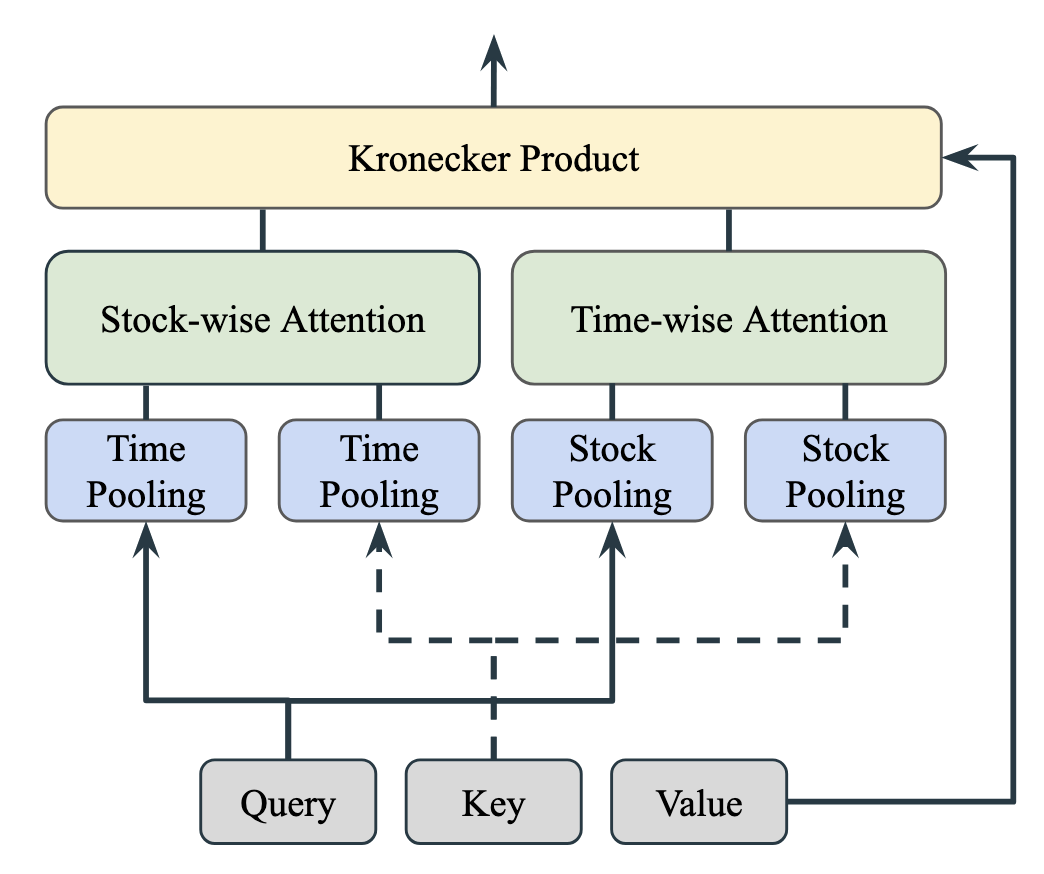}
    \caption{The overview of High Order Attention using Kronecker decomposition. }
    \label{fig:ho_att}
\end{figure}

\begin{theorem}
Given any fourth-order attention tensor \(\tensor{A} \in \R^{N \times N \times T \times T}\), which can be reshaped into a matrix \( A \in \R^{NT \times NT}\), there exists a rank \( R \) such that matrix \( A \) can be expressed as the sum of Kronecker products of matrices \( B_i \in \mathbb{R}^{N \times N} \) and \( C_i \in \mathbb{R}^{T \times T} \). Formally:
\begin{equation}
    A = \sum_{i=1}^R (B_i \otimes C_i),
\end{equation}
for some \( R \leq \min(N^2, T^2) \). As \( R \) approaches \(\min(N^2, T^2)\), the approximation becomes exact, meaning the Kronecker decomposition is capable of universally representing the original attention tensor \(\tensor{A}\) and matrix \( A \).
\end{theorem}


\begin{proof}
We begin by reshaping the tensor $\tensor{A}$ into a matrix $\Tilde{A} \in \mathbb{R}^{N^2 \times T^2}$. Consider   the Singular Value Decomposition (SVD) of $\Tilde{A}$:
\begin{equation}
    \Tilde{A} = U \Sigma V^T,
\end{equation}
where $U \in \mathbb{R}^{N^2 \times N^2}$ and $V \in \mathbb{R}^{T^2 \times T^2}$ are orthogonal matrices, and $\Sigma \in \mathbb{R}^{N^2 \times T^2}$ is a diagonal matrix containing the non-negative singular values of $\Tilde{A}$. This decomposition allows $\Tilde{A}$ to be represented as a sum of outer products of singular vectors scaled by the corresponding singular values, i.e.,
\begin{equation}
    \Tilde{A} = \sum_{i=1}^{min(N^2, T^2)} \sigma_i u_i v_i^T,
\end{equation}
where $\sigma_i$ are the singular values, and $u_i \in \mathbb{R}^{N^2}, v_i \in \mathbb{R}^{T^2}$ are the corresponding left and right singular vectors. For a low-rank approximation, only the first $R$ terms are retained:
\begin{equation}
    \Tilde{A} \approx \sum_{i=1}^R \sigma_i u_i v_i^T.
\end{equation}

Now, each product $\sigma_i u_i v_i^T$ can be reshaped from vectors $u_i$ and $v_i$ into matrices. Specifically, vector $u_i$ is reshaped into a matrix $\hat{U}_i \in \mathbb{R}^{N \times N}$, and $v_i$ into $\hat{V}_i \in \mathbb{R}^{T \times T}$. The outer product $\sigma_i u_i v_i^T$ corresponds to the Kronecker product of $\hat{U}_i$ and $\hat{V}_i$ when flattened:
\begin{equation}
    \sigma_i u_i v_i^T \rightarrow \sigma_i (\hat{U}_i \otimes \hat{V}_i).
\end{equation}

The matrices $B_i = \sqrt{\sigma_i} \hat{U}_i$ and $C_i = \sqrt{\sigma_i} \hat{V}_i$ are then defined, so that $\sigma_i (\hat{U}_i \otimes \hat{V}_i)$ simplifies to $B_i \otimes C_i$. This allows us to express the approximate decomposition of $\Tilde{A}$ as:
\begin{equation}
    \Tilde{A} \approx \sum_{i=1}^R B_i \otimes C_i,
\end{equation}
where $B_i \in \mathbb{R}^{N \times N}$ and $C_i \in \mathbb{R}^{T \times T}$. As $R$ approaches the rank of $\Tilde{A}$, this approximation converges to an exact representation. Thus, we verify the theorem's claim that any attention matrix derived from tensor $\tensor{A}$ can be expressed in terms of Kronecker products, achieving universal representability at full rank.
\end{proof}


Now we delve into the computation of the lower-order attention matrices \(S^{(1)}\) and \(S^{(2)}\). For simplicity, we omit the index \(i\). We project the input tensor \(\tensor{X}\) onto the query, key, and value tensors using mode-3 tensor products: \(\tensor{Q} = \tensor{X} \times_3 W^Q\), \(\tensor{K} = \tensor{X} \times_3 W^K\), and \(\tensor{V} = \tensor{X} \times_3 W^V\) respectively. Utilizing permutation-invariant pooling functions \(f: \mathbb{R}^{N \times T \times d} \to \mathbb{R}^{N \times d}\) and \(g: \mathbb{R}^{N \times T \times d} \to \mathbb{R}^{T \times d}\), we define the attention matrices as:
\begin{align}
    \label{eq:first-order-att}
    S^{(1)} &= \text{Softmax}\left(\frac{f(\tensor{Q}) f(\tensor{K})^T}{\sqrt{d}}\right),\\
    S^{(2)} &= \text{Softmax}\left(\frac{g(\tensor{Q}) g(\tensor{K})^T}{\sqrt{d}}\right).
\end{align}
Common examples of permutation-invariant pooling functions include sum, mean, and product.

The attention matrices are then used to modulate the value tensor \(\tensor{V}\). To maintain computational efficiency, we apply the attention matrices sequentially as follows:
\begin{equation}
    \label{eq:ho_attn}
    g_\text{Attn}(\tensor{Q}, \tensor{K}, \tensor{V}) = \tensor{V} \times_1 S^{(1)} \times_2 S^{(2)}
\end{equation}

Computing and applying these attention matrices incurs a computational cost of \(\mathcal{O}(d(N^2T + T^2N))\), where the quadratic terms reflect the inherent computational demand of the scaled dot-product attention mechanism, which can become substantial. To mitigate these costs, we employ kernelized linear attention, which achieves linear computational complexity of \(\mathcal{O}(d^2NT)\).\\

\subsubsection{Linear Attention With Kernel Trick\\}

Following the work by \cite{choromanski2022rethinking}, we approximate the attention matrix \(S^{(1)}\) in equation \eqref{eq:first-order-att} using a suitable kernel function \(\phi: \mathbb{R}^{N \times d} \to \mathbb{R}^{N \times d}\) as:
\begin{equation}
    S^{(1)} \approx Z_1^{-1} \phi(f(\tensor{Q})) \phi(f(\tensor{K}))^T,
\end{equation}
where \(Z_1 \in \mathbb{R}^{N \times N}\) is a diagonal matrix serving as a normalizing factor. The diagonal elements of \(Z_1\) are defined by:
\begin{equation}
    (Z_1)_{ii} = \phi(f(\tensor{Q}))_i \sum_{j \in [N]} \phi(f(\tensor{K}))_j^T,
\end{equation}
Substituting into equation \eqref{eq:ho_attn}, we compute:
\begin{equation}
    (\tensor{V} \times_1 S^{(1)})_{:t:} = Z_1^{-1} \phi(g_1(\tensor{Q})) \phi(g_1(\tensor{K}))^T \tensor{V}_{:t:},
\end{equation}
where \(\tensor{V}_{:t:}\) denotes the tensor slice corresponding to the \(t\)-th timestep. The computational sequence initiates by calculating \(\phi(f(\tensor{K}))^T \tensor{V}_{:t:}\) followed by the subsequent operations, yielding a computational complexity of \(\mathcal{O}(d^2 NT)\). The choice of kernel function \(\phi\) is flexible, and we utilize the same kernel function as in \cite{choromanski2022rethinking}, which has been validated both theoretically and empirically. Similarly, the same computational approach is applied to compute \(S^{(2)}\) using the same kernel function \(\phi\).\\

\subsection{Model Architecture}
\label{sec:arch}
Our model's architecture consists of a multilayer transformer network. Input tensors are transformed through a linear projection layer to align the features with the hidden dimensions required by the model and its attention modules. We adopt pre-normalization techniques, specifically RMSNorm \citep{zhang2019root}, in every layer following the approach suggested by \citet{touvron2023llama}. Rotary Positional Embedding \cite{su2023roformer} is applied to the query and key matrices exclusively for computing the temporal attention \(S^{(2)}\); stock-wise attention does not involve positional embeddings as the ordering in this dimension is not meaningful. 

 Inspired by \cite{hu2021unit, tsai2019multimodal}, our proposed multimodal model follows an encoder-decoder architecture where the modality of the data differs across encoder and decoder. More specifically, text encodings are processed by the transformer encoder, and price timeseries data by the transformer decoder as presented in figure \ref{fig:architecture}. Cross-attention layers in the network facilitate information blending between these two modalities.\\

\begin{figure}
    \centering
    \includegraphics[width=\linewidth]{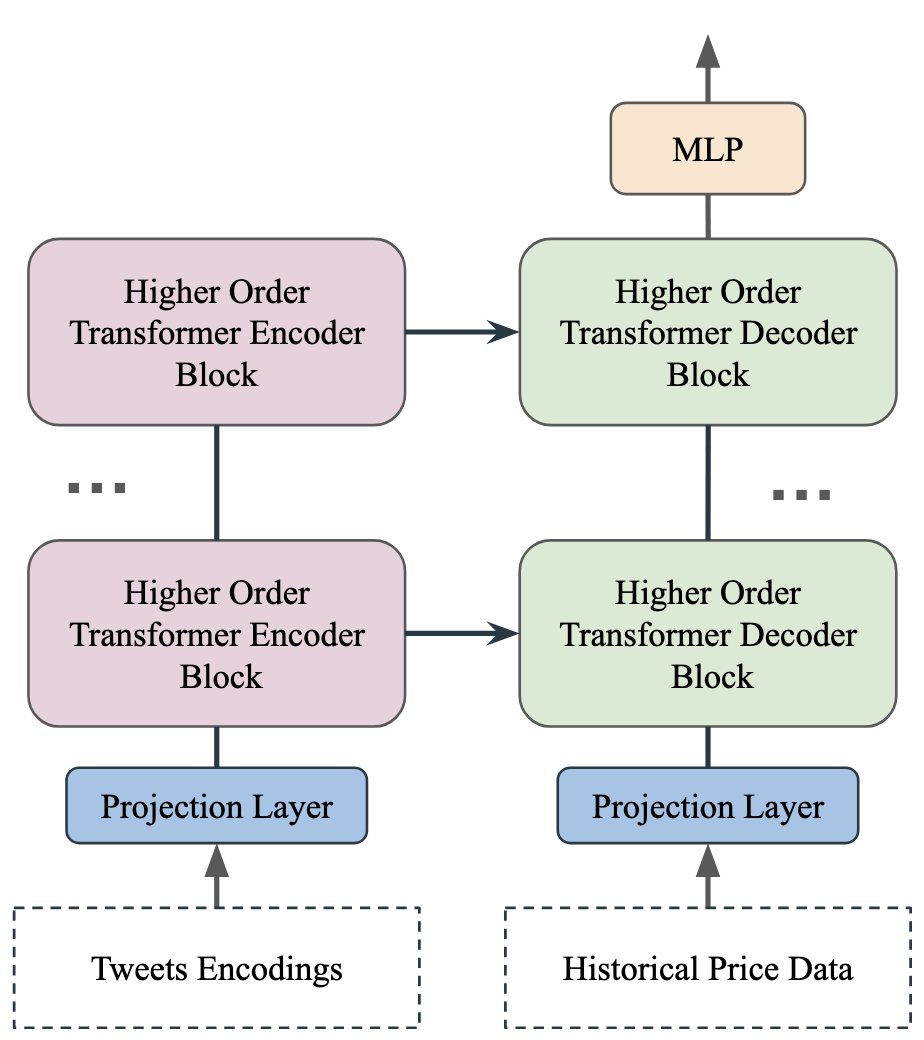}
    \caption{Multimodal transformer architecture. As depicted in the figure, tweet encodings are fed to the transformer encoder, and the historical price data are given to the transformer decoder. }
    \label{fig:architecture}
\end{figure}

\section{Experiments}
\subsection{Dataset}
We demonstrate the capability of Higher Order Transformers on stock market movement classification using a multimodal dataset called Stocknet \cite{xu2018stock} which comprises historical data from 88 stocks extracted from Yahoo Finance\footnote{https://finance.yahoo.com/} and related news crawled from Twitter over two years. We adopt the data processing methodology as outlined in \cite{sawhney2020deep, xu2018stock}, shifting a 5-day lag window along trading days to generate samples. Samples are labeled based on the movement percentage of the closing price, with movements \(\geq 0.55\%\) labeled as positive and \(\leq 0.5\%\) as negative. Samples lacking either prices or tweets on trading days are discarded to ensure data consistency.

Text encodings are generated using the FinBert model \cite{araci2019finbert} by processing all tweets for each stock on each day through the transformer encoder, then averaging over tokens and tweets to produce a vector representation for each stock per day.

The dataset is temporally divided in a 70:10:20 ratio for train, validation, and test splits, spanning from 01/01/2014 to 31/07/2015 for training, 01/08/2015 to 30/09/2015 for validation, and 01/10/2015 to 01/01/2016 for testing.

\subsection{Model Configuration}
We configured the models using the Adam optimizer with an initial learning rate of \(0.0001\). Tweet embeddings from FinBert are dimensioned at 768. Training extends for up to 1000 epochs with an early stopping criterion based on validation F1 performance, stopping further training when no improvement is observed. Hyperparameter tuning is conducted through a grid search with the following parameters:
\begin{itemize}
    \setlength{\itemsep}{0pt}
    \item Hidden dimension: [32, 64, 128]
    \item Number of attention heads: [1, 4, 8, 16]
    \item Number of transformer blocks: [2, 4, 6]
    \item Dropout probability: [0, 0.2, 0.4]
\end{itemize}
Lastly, we report the performance of the best model based on validation F1 on the test set.

\subsection{Metrics}
In line with prior research in stock prediction \cite{sawhney2020deep, xu2018stock}, we evaluate classification performance using three metrics: accuracy, F1 score, and Matthew's correlation coefficient (MCC) \footnote{Using Torchmetrics library: \url{https://github.com/Lightning-AI/torchmetrics}}. We particularly include MCC as it provides a more balanced performance measure that accounts for both positive and negative classes and adjusts for any class imbalance by incorporating True Negatives. This makes it especially valuable in contexts where the data may be skewed. MCC is defined as:

\begin{equation}
     MCC = \frac{(TP \times TN) - (FP \times FN)}{\sqrt{(TP + FP)(TP + FN)(TN + FP)(TN + FN)}}
\end{equation}

This formula captures the correlation between observed and predicted binary classifications and is regarded as a reliable statistical rate that can deliver a high score only if the predictor successfully predicts the majority of positive and negative instances correctly.

\subsection{Baselines}
We compare our results with the following baselines:
\begin{itemize}
    \setlength{\itemsep}{5pt}
    \item \textbf{RAND}: A simple predictor to make random guesses about the rise and fall.
    \item \textbf{ARIMA}: AutoRegressive Integrated Moving Average, an advanced technical analysis method using only price signals. \cite{brown2004smoothing}
    \item \textbf{HAN}: Hybrid Attention Networks, is a model developed to predict stock trends by analyzing sequences of recent news related to the market. It employs attention mechanisms and self-paced learning to handle the chaotic nature of online content effectively.\cite{hu2019listening}
    \item \textbf{StockNet}: A Variational Auto-Encoder that jointly exploits text and price signals to predict stock movement. \cite{xu2018stock}
    \item \textbf{MAN-SF}: Multipronged Attention Network uses tweets, price history, and inter-stock connections to train a GAT network that captures relevant signals across modalities. \cite{sawhney2020deep}
    \item \textbf{HATS}: A hierarchical graph attention network for stock movement prediction. \cite{kim2019hats}
    \item \textbf{Yangjia et al.}: An RNN-based model which uses a co-attention mechanism to capture the sufficient context information between text and prices across every day. \cite{yangjia-etal-2022-fundamental}
    \item \textbf{ChatGPT (Zero-Shot)}: Using proper prompting on ChatGPT in a zero-shot setting. \cite{xie2023wall}
    \item \textbf{CapTE}: A Transformer Encoder that extracts the deep semantic features of the social media and then captures the structural relationship of the texts through a capsule network. \cite{liu-etal-2019-transformer}
    \item \textbf{TEANet}: A Transformer encoder with multiple attention mechanisms that extracts features and fuses text data from social media with stock prices. \cite{ZHANG2022117239}

    \item \textbf{SLOT}: A self-supervised learning method that embeds stocks and tweets in a shared semantic space to predict stock price movements, overcoming tweet sparsity and noise issues. \cite{slot}
    \item \textbf{NL-LSTM}: A novel model that uses neutrosophic logic for more accurate sentiment analysis of social media data, combined with LSTM for forecasting based on sentiment and historical stock data. \cite{jtaer19010007}
\end{itemize}

\begin{table}[t]
  \centering
  \caption{Classification performance of our method and competitors. The best is in bold, and the second best is underlined. Our model shows the best performance in all evaluation metrics after NL-LSTM.}
  \label{tab:benchmark}
  \resizebox{\linewidth - 40px}{!}{
  \begin{tabular}{lccc}
    \toprule
    Method & ACC & F1 & MCC \\
    \midrule
    RAND & 50.9 & 50.2 & -0.002 \\
    ARIMA \citeyearpar{brown2004smoothing} & 51.4 & 51.3 & -0.021 \\
    $\text{ChatGPT}_{zs}$ \citeyearpar{xie2023wall} & 52.1 & 51.7 & 0.038 \\
    HATS \citeyearpar{kim2019hats} & 56.02 & 56.22 & 0.117 \\
    HAN \citeyearpar{hu2018listening} & 57.6 & 57.2 & 0.052 \\
    StockNet \citeyearpar{xu2018stock} & 58.2 & 57.5 & 0.081 \\
    SLOT \citeyearpar{slot} & 58.7 & - & 0.2065 \\
    MAN-SF \citeyearpar{sawhney2020deep} & 60.8 & 60.5 & 0.195 \\
    \citeauthor{yangjia-etal-2022-fundamental}  \citeyearpar{yangjia-etal-2022-fundamental} & 62.6 & 61.1 & 0.228 \\
    CapTE \citeyearpar{liu-etal-2019-transformer} & 64.22 & - & 0.348\\
    TEANet \citeyearpar{ZHANG2022117239} & 65.16 & - & 0.364\\
    \textbf{NL-LSTM} \citeyearpar{jtaer19010007} & \textbf{78.5} & - & \textbf{0.587}\\
    \hline
    \underline{Ours} & \underline{72.94} & \underline{72.23} & \underline{0.516} \\
    \bottomrule
  \end{tabular}
  }
\end{table}

\subsection{Results}
In this section, we analyze the benchmark performance of our model against various baseline models on the StockNet dataset. The results, as summarized in Table \ref{tab:benchmark}, demonstrate the superior performance of our model over all of the existing baselines except NL-LSTM, which significantly outperforms the rest of the methods in terms of accuracy, F1, and MCC. NL-LSTM reported the highest accuracy performance on binary stock movement prediction by incorporating fuzzy logic in sentiment analysis \cite{jtaer19010007}.

\begin{table}[t]
  \centering
  \caption{Ablation study on the data modality and attention study shows the effectiveness of using multimodal over unimodal data, and using kernelized over standard attention.}
  \label{tab:ablation_mod_type}
  \begin{tabular}{lcccc}
    \toprule
    Modality & Attention Type & ACC & F1 & MCC \\
    \midrule
    \multirow{2}{*}{Price} & Standard & 52.3 & 51.8 & 0.049 \\
     & Kernelized & 53.6 & 53.3 & 0.063 \\
    \hline
    \multirow{2}{*}{Tweet} & Standard & 69.5 & 68.0 & 0.320 \\
     & Kernelized & 71.0 & 70.1 & 0.402 \\
    \hline
    \multirow{2}{*}{Multimodal} & Standard & 70.8 & 70.5 & 0.391 \\
     & \textbf{Kernelized} & \textbf{72.9} & \textbf{72.2} & \textbf{0.516} \\
    \bottomrule
\end{tabular}
\end{table}

\subsection{Ablation Study}
We further investigate the impact of different aspects of our model through an ablation study, focusing on the types of attention mechanisms used, the data modalities, and the attention method. The results are provided in Tables \ref{tab:ablation_mod_type}, and \ref{tab:ablation_att} respectively.

The influence of data modality on performance is depicted in Table \ref{tab:ablation_mod_type}. The multimodal approach that integrates both price data and Twitter news significantly outperforms single-modality approaches, underscoring the benefit of leveraging diverse data sources. Moreover, the text-based models performs better than the timeseries-based models with a significant gap, showing the rich context present in the news data crawled from Twitter for the stock movement prediction task.

Table \ref{tab:ablation_mod_type} also explores the effect of using linear versus standard attention mechanisms over different modalities. The results highlight the advantages of linear attention in terms of efficiency and effectiveness, particularly in multimodal settings.

\begin{table}[ht]
  \centering
  \caption{Ablation study on the attention dimension shows the effectiveness of attention across both dimensions. }
  \label{tab:ablation_att}
  \begin{tabular}{lccc}
    \toprule
    Attention Dimension & ACC & F1 & MCC \\
    \midrule
    None & 53.5 & 52.4 & 0.069 \\
    Stock-wise & 71.3 & 69.2 & 0.495 \\
    Time-wise & 72.1 & 71.3 & 0.507 \\
    \textbf{Both (Higher Order)} & \textbf{72.9} & \textbf{72.2} & \textbf{0.516} \\
    \bottomrule
  \end{tabular}
\end{table}

Table \ref{tab:ablation_att} vividly demonstrates that while applying attention in any single dimension (either stock-wise or time-wise) improves performance metrics compared to using no attention, the most significant enhancements are observed when attention is simultaneously applied across both dimensions.

\section{Conclusion}
In this paper, we presented the Higher Order Transformers, a novel architecture tailored to predict stock movements by processing multimodal stock data. By expanding the self-attention mechanism and transformer architecture to incorporate higher-order interactions, our model adeptly captures the intricate dynamics of financial markets over both stock and time. To address computational constraints, we implemented low-rank approximations through tensor decomposition and integrated kernel attention to achieve linear computational complexity. Extensive testing on the Stocknet dataset demonstrated that our approach significantly surpasses most of the existing models in predicting stock movements. An ablation study further validated the effectiveness of specific architectural components, highlighting their contributory value to the model's performance. Looking ahead, we plan to train our model on other multimodal stock datasets such as ASTOCK \cite{zou2022astock} and Dhaka Stock Exchange \cite{Muhammad_2023}, and perform profitability analysis on real-world stock data to further test the practical application and financial viability of our proposed method.

\begin{acks}

\end{acks}

\bibliographystyle{ACM-Reference-Format}
\bibliography{sample-base}

\end{document}